\documentclass{article}

\usepackage[preprint,nonatbib]{neurips_2020}
\usepackage[utf8]{inputenc} 
\usepackage[T1]{fontenc}    
\usepackage{hyperref}       
\usepackage{url}            
\usepackage{booktabs}       
\usepackage{amsfonts}       
\usepackage{nicefrac}       
\usepackage{microtype}      
\usepackage{amsmath}
\usepackage{amssymb}
\usepackage{amsthm}
\usepackage{bm}
\usepackage{bbm}
\usepackage{graphicx}
\usepackage{dblfloatfix}
\usepackage[dvipsnames]{xcolor}
\newcommand{\RED}[1]{{\color{black}#1}}
\newcommand{\diag}{\mbox{diag}}

\newtheorem{theorem}{Theorem}
\newtheorem{lemma}{Lemma}
\newtheorem{corollary}{Corollary}
\newtheorem{remark}{Remark}

\title{Robust Mean Estimation in High Dimensions via $\ell_0$ Minimization}

\author{%
  Jing~Liu \\
 Coordinated Science Laboratory\\
  University of Illinois at Urbana-Champaign\\
  Urbana, IL 61801 \\
  \texttt{jil292@illinois.edu} \\
   \And
  Aditya~Deshmukh \\
 Coordinated Science Laboratory\\
  University of Illinois at Urbana-Champaign\\
  Urbana, IL 61801 \\
  \texttt{ad11@illinois.edu} \\
  \AND
 Venugopal~V.~Veeravalli\\
 Coordinated Science Laboratory\\
  University of Illinois at Urbana-Champaign\\
  Urbana, IL 61801 \\
  \texttt{vvv@illinois.edu} \\
}

\begin{document}

\maketitle

\begin{abstract}
  We study the robust mean estimation problem in high dimensions, where $\alpha <0.5$ fraction of the data points can be arbitrarily corrupted. 
  Motivated by compressive sensing, we formulate the robust mean estimation problem as the minimization of the  $\ell_0$-`norm' of the \emph{outlier indicator vector}, under second moment constraints on the inlier data points. We prove that the global minimum of this objective is order optimal for the robust mean estimation problem, and we propose a general framework for minimizing the objective.  
We further leverage the $\ell_1$ and $\ell_p$ $(0<p<1)$,  minimization techniques in compressive sensing to provide computationally tractable solutions to the $\ell_0$ minimization problem. Both synthetic and real data experiments demonstrate that the proposed algorithms significantly outperform state-of-the-art robust mean estimation methods. 
\end{abstract}

\section{Introduction}

Robust mean estimation in high dimensions has received considerable interest recently, and has found applications in areas such as data analysis and distributed learning. Classical robust mean estimation methods such as coordinate-wise median and geometric median have error bounds that scale with the dimension of the data \cite{LRV}, which results in poor performance in the high dimensional regime. A notable exception is Tukey's Median~\cite{Tukey1975MathematicsAT} that has an error bound that is independent of the dimension, when the fraction of outliers is less than a threshold 
~\cite{donoho1992breakdown,zhu2020does}. However, the computational complexity of Tukey's Median algorithm is exponential in the dimension. 

A number of recent papers have proposed polynomial-time algorithms that have dimension independent error bounds under certain distributional assumptions (e.g., bounded covariance or concentration properties). For a comprehensive survey on robust mean estimation, we refer the interested readers to~\cite{diakonikolas2019recent}.
One of the first such algorithms is Iterative Filtering~\cite{7782980,diakonikolas2017being,steinhardt2018robust}, in which one finds the top eigenvector of the sample covariance matrix and removes (or down-weights) the points with large projection scores on that eigenvector, and then repeats this procedure on the rest of points until the top eigenvalue is small. However, as discussed in~\cite{NIPS2019_8839}, the drawback of this approach is that it only looks at one direction/eigenvector at a time, and the outliers may not exhibit unusual bias in only one direction or lie in a single cluster. Figure~\ref{fig:LpApproximation} illustrates an example for which Iterative Filtering might have poor empirical performance. In this figure, the inlier data points in blue are randomly generated from the standard Gaussian distribution in high dimension $d$, and therefore their $\ell_2$-distances to the origin are all roughly $\sqrt{d}$. There are two clusters of outliers in red, and their $\ell_2$-distances to the origin are also roughly $\sqrt{d}$. If there is only one cluster of outliers, Iterative Filtering can effectively identify them, however, in this example, this method may remove many inlier points and perform suboptimally.

\begin{figure*}[!h]
	\centering
	\includegraphics[trim=280 0 0 0, clip=true,width=0.24\linewidth]{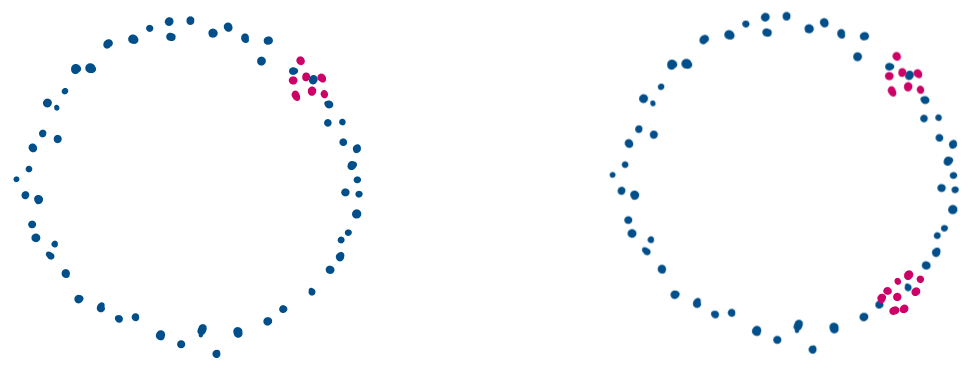}
	\caption{Illustration of two clusters of outliers (red points). The inlier points (blue) are drawn from standard Gaussian distribution in high dimension $d$. Both the outliers and inliers are roughly $\sqrt{d}$ distance from the origin.}
	\label{fig:LpApproximation}
\end{figure*}

There are interesting connections between existing methods for robust mean estimation and those used in Compressive Sensing (CS). The Iterative Filtering algorithm has similarities to the greedy Matching Pursuit type CS  algorithm~\cite{Mallat1993MatchingPW}. In the latter algorithm, one finds a single column of sensing matrix $\bm A$ that has largest correlation with the measurements $\bm b$, removes that column and its contribution from $\bm b$, and repeats this procedure on remaining columns of $\bm A$. 
In light of this, we expect Iterative Filtering to have poor empirical performance despite having order optimality guarantees.
Dong et al.~\cite{NIPS2019_8839} proposed a new scoring criteria for finding outliers, in which one looks at multiple directions associated with large eigenvalues of the sample covariance matrix in every iteration of the algorithm. Interestingly, this approach is conceptually similar to  Iterative Thresholding techniques in CS (e.g., Iterative Hard Thresholding~\cite{blumensath2009iterative} or Hard Thresholding Pursuit~\cite{foucart2011hard}), in which one simultaneously finds multiple columns of matrix $\bm A$ that are more likely contribute to $\bm b$. Although this type of approach is also greedy, it is more accurate than the Matching Pursuit technique in practice. 

A common assumption in robust mean estimation problem is that the fraction of the corrupted data points is small. In this paper, we explicitly use this information 
through the introduction of an  \textit{outlier indicator vector} whose  $\ell_0$-`norm' we minimize under distributional constraints on the uncorrupted data points. This new formulation enables us to leverage the well-studied CS techniques to solve the robust mean estimation problem.

We consider the setting wherein the distribution of the uncorrupted data points has bounded covariance, as is commonly assumed in many recent works (e.g.,~\cite{diakonikolas2017being,NIPS2019_8839,convex,steinhardt2017resilience}). 
In particular, in ~\cite{convex}, in addition to assuming the bounded covariance constraints on the uncorrupted data points, the authors also assume that the fraction $\alpha$ of the outlier points is known. They propose to minimize the spectral norm of the weighted sample covariance matrix and use the knowledge of outlier fraction $\alpha$ to constrain the weights. In contrast, we do \textit{not} assume the knowledge of the outlier fraction $\alpha$, which is usually not known in practice. Additionally, our strategy to update the estimate of the mean is more straightforward and different from that given in ~\cite{convex}.

\textcolor{black}{Lastly, we remind interested readers that there is another line of related works on mean estimation of heavy tailed distributions. See a recent survey~\cite{lugosi2019mean} and the references therein.} 


\paragraph{Contributions}
\begin{itemize}
\item At a fundamental level, a key contribution of this paper is the formulation of the robust mean estimation problem as minimizing the $\ell_0$-`norm' of the proposed \emph{outlier indicator vector}, under distributional constraints on the uncorrupted data points. 
%
\item We provide a theoretical justification for this novel objective. We further propose a general iterative framework for minimizing this objective, which will terminate in a finite number of iterations. 
\item Under this formulation, we are able to leverage  powerful $\ell_p  (0<p\leq 1)$ minimization techniques from Compressive Sensing to solve the robust mean estimation problem. We demonstrate via simulations that our algorithms significantly outperform the state-of-the-art methods in robust mean estimation.
\end{itemize}
\section{Objective function}
We begin by introducing some definitions and notation.
\newtheorem{definition}{Definition}
\begin{definition}\label{def1}
($\alpha$-corrupted samples~\cite{NIPS2019_8839}) Let $P$ be a distribution on $\mathbb{R}^d$ with unknown mean $\bm x^*$. We first have $\tilde{\bm y}_1,...,\tilde{\bm y}_n$ i.i.d. drawn from $P$, then modified by an adversary who \textcolor{black}{can} inspect all the samples, remove $\alpha n$ of them, and replace them with arbitrary vectors in $\mathbb{R}^d$, then we get an $\alpha$-corrupted set of samples, denoted as $\bm y_1,...,\bm y_n$.  
\end{definition}

There are other types of contamination one can consider, for e.g., Huber's $\epsilon$-contamination model~\cite{huber1964}. The contamination model described in Definition \ref{def1} is the strongest in the sense that the adversary is not oblivious to the original data points, and can replace any subset of $\alpha n$ data points with any vectors in $\mathbb{R}^d$. We refer the reader to~\cite{diakonikolas2019recent} for a more detailed discussion on contamination models.
\begin{definition}
(Resilience~\cite{steinhardt2017resilience}) A set of points $\bm y_1,...,\bm y_m$ lying in $\mathbb{R}^d$ is ($\sigma$,$\beta$)-resilient in $\ell_2$-norm around a point $\bm x$ if, for all its subsets $\bm T$ of size at least $(1-\beta)m$, $\left \|\frac{1}{|\bm T|}\sum_{\bm y_i\in \bm T} \bm y_i - \bm x \right \|_2 \leq \sigma$.  
\end{definition}

Our primary goal is to robustly estimate the mean of the uncorrupted data points given a set of $\alpha$-corrupted samples. To explicitly utilize the knowledge that the fraction of the corrupted points is small, we introduce an \textit{outlier indicator vector} $\bm h \in \mathbb{R}^n$: for the $i$-th data point, $h_i$ indicates that whether it is an outlier ($h_i \neq 0$) or not ($h_i= 0$). We minimize the $\ell_0$-`norm' of $\bm h$ under a second moment constraint on the inlier points.

Here we only impose a second moment assumption. Since we did not make any assumption on further higher moments, it may be possible for a few uncorrupted samples to affect the empirical covariance too much. Fortunately,~\cite{diakonikolas2017being} shows that such troublesome uncorrupted samples have only a small probability to occur:

Let $0<\epsilon<1$ be fixed. Let $\bm S=\lbrace\tilde{\bm y}_1,\dots,\tilde{\bm y}_n\rbrace$ be a set of $n\geq 3.2\times10^4\frac{d \log d}{\epsilon}+11.2\times10^4\frac{d}{\epsilon}$ samples drawn from a distribution $P$ with mean $\bm x^*$ and covariance matrix $\preceq \sigma^2 I$. Let $\bm G \triangleq \{\tilde{\bm y}_i|\|\tilde{\bm y}_i - \bm x^* \|_2 \leq \sigma\sqrt{40d/\epsilon}\}$ denote the number of samples which are less than $\sigma\sqrt{40d/\epsilon}$ distance away from $\bm x^*$. It follows from~\cite[Lemma A.18 (ii)]{diakonikolas2017being} that
\begin{equation}\label{eq:prob1}
    \mathrm{Pr}\left(|\bm G|\geq n- \epsilon n\right) \geq 39/40.
\end{equation}
We consider the \textit{far away} uncorrupted samples $\bm S\setminus \bm G$ as outliers also, without sacrificing performance significantly. Note that it is also possible to remove such samples through preprocessing~\cite{diakonikolas2019robust,NIPS2019_8839,convex}.


 Let $\lbrace \bm y_1,\dots,\bm y_n \rbrace$ be an $\alpha$-corrupted set of $\bm S$. Let $\bm h^*$ be such that $h_i^*=1$ for the outliers (both far away uncorrupted samples and corrupted samples), and $h_i^*=0$ for the rest of uncorrupted data points. Let $E$ be the event:

\begin{equation}\label{event}
    E = \left\lbrace \lambda_{\mathrm{max}}\left(\sum_{\tilde{\bm y}_i\in \bm G} (\tilde{\bm y}_i-\bm x^*)(\tilde{\bm y}_i-\bm x^*)^\top\right)\leq \frac{3}{2} n  \sigma^2 \right\rbrace.
\end{equation}

Note that the set of \textit{inliers} satisfies $\{\bm y_i | h_i^*=0\}=\{\tilde{\bm y}_i | h_i^*=0\}\subseteq \bm G$. 
Since $(\tilde{\bm y}_i-\bm x^*)( \tilde{\bm y}_i-\bm x^*)^\top$ is PSD, we must have 
\begin{align*}
    \lambda_{\mathrm{max}}\left(\sum_{i=1}^{n} (1-h_i^*)(\bm y_i-\bm x^*)(\bm y_i-\bm x^*)^\top\right)\leq \lambda_{\mathrm{max}}\left(\sum_{\tilde{\bm y}_i\in \bm G} (\tilde{\bm y}_i-\bm x^*)(\tilde{\bm y}_i-\bm x^*)^\top\right).
\end{align*}
This implies
\begin{equation}\label{event2}
    \left\lbrace\lambda_{\mathrm{max}}\left(\sum_{i=1}^{n} (1-h_i^*)(\bm y_i-\bm x^*)(\bm y_i-\bm x^*)^\top\right)\leq \frac{3}{2} n \sigma^2\right\rbrace\supseteq E.
\end{equation}
Then, we have:
\begin{equation}\label{goodPoints_cov}
\mathrm{Pr}\left\lbrace\lambda_{\mathrm{max}}\left(\sum_{i=1}^{n} (1-h_i^*)(\bm y_i-\bm x^*)(\bm y_i-\bm x^*)^\top\right)\leq \frac{3}{2} n \sigma^2\right\rbrace
\geq\mathrm{Pr}(E) \geq 39/40,
\end{equation}
where the last inequality follows from~\cite[Lemma A.18 (iv)]{diakonikolas2017being}.

These motivate us to propose the following objective:
\begin{align}\label{obj}
   \min_{\bm x, \bm h} \|\bm h\|_0 \quad
    s.t.\  &  0\leq h_i \leq 1, \forall i,\\
   & \lambda_{\mathrm{max}}\left(\sum_{i=1}^{n} (1-h_i)(\bm y_i-\bm x)(\bm y_i-\bm x)^\top \right)\leq c_1^2 n \sigma^2. \nonumber
\end{align}
Our intended solution is to have $h_i=0$ for the inlier points and $h_i=1$ for the outlier points.
The intuition behind the second moment constraint is based on the following key insight identified in previous works (e.g., ~\cite{diakonikolas2017being}): if the outliers shift the mean by $\Omega(\Delta)$, then they must shift the spectral norm of the covariance matrix by $\Omega(\Delta^2/\alpha)$. Notice that in~\eqref{goodPoints_cov}, the constant $\frac{3}{2}$ is based on $n=\Omega(\frac{d \log d}{\epsilon})$ samples. In the constraint of the proposed objective~\eqref{obj}, we use a general $c^2_1$ instead of $\frac{3}{2}$, and $c^2_1$ should be no less than $\frac{3}{2}$.


We first provide a theoretical justification for the $\ell_0$ minimization objective, and then give a general framework for solving \eqref{obj}, thereby obtaining a robust estimate of the mean. Before proceeding, we first introduce a lemma which is based on~\cite[Section 1.3]{steinhardt2017resilience}:

\begin{lemma}\label{boundedcovariance_resilience}
For a set of data points $\bm S\triangleq\{\bm y_i\}$, \textcolor{black}{let $\bm x =\frac{1}{|\bm S|}\sum_{\bm y_i \in \bm S} \bm y_i$}. If $\lambda_{\mathrm{max}}\left(\frac{1}{|\bm S|}\sum_{\bm y_i \in \bm S} (\bm y_i-\bm x)(\bm y_i-\bm x)^\top \right)\leq \sigma^2$, then the set $\bm S$ is $\left (2\sigma \sqrt{\beta},\beta \right)$-resilient in $\ell_2$-norm around $\bm x$ for all $\beta<0.5$.\\
\end{lemma}

We now provide theoretical guarantees for the estimator which is given by the solution of the optimization problem \eqref{obj}. Assume that $\epsilon$ is fixed, which controls the estimator's error. We show that given $\alpha$-corrupted $\Omega\left(\frac{d\log d}{\epsilon}\right)$ samples, with high probability the $\ell_2$-norm of the estimator's error is bounded by $O(\sigma\sqrt{\alpha+\epsilon})$. We formalize this in the following theorem.\\
\begin{theorem}\label{guarantee}
Let $P$ be a distribution on $\mathbb{R}^d$ with unknown mean and unknown covariance matrix 	$\preceq \ \sigma^2 I$. Let $0<\epsilon<1/3$ be fixed. Let $0<\alpha < 1/3-\epsilon$. Given an $\alpha$-fraction corrupted set of $n\geq 3.2\times 10^4\frac{d \log d}{\epsilon}+11.2\times 10^4 \frac{d}{\epsilon}$ samples from $P$, and set $c^2_1\geq \frac{3}{2}$ in \eqref{obj}. With probability at least $0.95$, the globally optimal solution $(\bm h^{\mathrm{opt}}, \bm x^{\mathrm{opt}})$ of \eqref{obj} with $h_i^{\mathrm{opt}} \in \{0,1\}$ satisfies $\| \bm x^{\mathrm{opt}}-\textcolor{black}{\Bar{\bm x}^*}\|_2 \leq (4+3c_1)\sigma \sqrt{\alpha+\epsilon}$, \textcolor{black}{and any feasible solution $(\hat{\bm h}, \hat{\bm x})$ with ${\hat{h}_i} \in \{0,1\}$ and $\|\hat{\bm h}\|_0 \leq (\alpha+\epsilon)n$ satisfies $\| \hat{\bm x}-{ \Bar{\bm x}^*}\|_2 \leq (4+3c_1)\sigma \sqrt{\alpha+\epsilon}$.} \textcolor{black}{Where $\Bar{\bm x}^*$ is the average of inlier points corresponding to $h_i^*=0$ defined in \eqref{event}.}
\end{theorem}

\begin{proof}
Let $(\bm h^{\mathrm{opt}}, \bm x^{\mathrm{opt}})$ be the global optimal solution of \eqref{obj} with $h_i^{\mathrm{opt}} \in \{0,1\}$. Then 
$
\bm x^{\mathrm{opt}}=\frac{\sum_{i=1}^{n}(1-h_i^{\mathrm{opt}})\bm y_i}{\sum_{i=1}^{n}(1-h_i^{\mathrm{opt}})}$, i.e., $\bm x^{\mathrm{opt}}$ is the average of the $\bm y_i$'s corresponding to $h_i^{\mathrm{opt}}=0$. Note that for any global optimal solution of \eqref{obj}, by setting its non-zero $h_i$ to be 1, we can always get corresponding feasible $(\bm h^{\mathrm{opt}}, \bm x^{\mathrm{opt}})$ with $h_i^{\mathrm{opt}} \in \{0,1\}$, and the objective value remains unchanged.

Consider $\bm h^*$ as defined in~\eqref{event}. Let $\alpha'\triangleq \epsilon+\alpha< 1/3$. Let $\tilde E = \lbrace\|\bm h^*\|_0 \leq \alpha' n\rbrace \cap E$, where $E$ is as defined in \eqref{event}. It follows from \eqref{eq:prob1} and \eqref{goodPoints_cov} that
\begin{equation}
    \mathrm{Pr}(\tilde E)\geq 0.95.
\end{equation}


Then on the event $\tilde E$, it follows from \eqref{event2} and \textcolor{black}{the fact that $\lambda_{\mathrm{max}}\left(\sum_{i=1}^{n} (1-h_i^*)(\bm y_i-\Bar{\bm x}^*)(\bm y_i-\Bar{\bm x}^*)^\top\right) \leq \lambda_{\mathrm{max}}\left(\sum_{i=1}^{n} (1-h_i^*)(\bm y_i-\bm x^*)(\bm y_i-\bm x^*)^\top\right)$}, that the set $\bm S^* \triangleq \{\bm y_i | h_i^*=0\}$ is $\left (3\sigma \sqrt{\beta},\beta \right)$-resilient in $\ell_2$-norm around \textcolor{black}{$\Bar{\bm x}^*$} for all $\beta<0.5$ by Lemma~\ref{boundedcovariance_resilience}. We also have $|\bm S^*|\geq(1-\alpha')n$.

Since $(\bm h^{\mathrm{opt}}, \bm x^{\mathrm{opt}})$ is globally optimal, and $(\bm h^{*}, \Bar{\bm x}^{*})$ is feasible, we have $\|\bm h^{\mathrm{opt}}\|_0 \leq \|\bm h^*\|_0 \leq \alpha' n$. Thus $n-\|\bm h^{\mathrm{opt}}\|_0 \triangleq q \geq n-\alpha' n$. Note that
\begin{align}\label{covariance_opt}
\lambda_{\mathrm{max}}\left(\sum_{i=1}^{n} (1-h_i^{\mathrm{opt}})(\bm y_i-\bm x^{\mathrm{opt}})(\bm y_i-\bm x^{\mathrm{opt}})^\top\right)\leq {c^2_1  n \sigma^2}. 
\end{align}

Normalizing~\eqref{covariance_opt} by $q$ leads to
\begin{align}\label{covariance_opt_normalize}
\lambda_{\mathrm{max}}\left(\frac{1}{q} \sum_{i=1}^{n} (1-h_i^{\mathrm{opt}})(\bm y_i-\bm x^{\mathrm{opt}})(\bm y_i-\bm x^{\mathrm{opt}})^\top\right)\leq {c^2_1  \frac{n}{q} \sigma^2}\leq {c^2_1  \frac{n}{ n-\alpha' n} \sigma^2}. 
\end{align}

Because $h_i^{\mathrm{opt}} \in \{0,1\}$, \eqref{covariance_opt_normalize} implies that the set $\bm S^{\mathrm{opt}} \triangleq \{\bm y_i | h_i^{\mathrm{opt}}=0\}$ is $\left (c_1\sqrt{6}\sigma \sqrt{\beta},\beta \right)$-resilient in $\ell_2$-norm around $\bm x^{\mathrm{opt}}$ for all $\beta<0.5$ by Lemma~\ref{boundedcovariance_resilience}. We also have $|\bm S^{\mathrm{opt}}|\geq(1-\alpha')n$.

Let $\bm T \triangleq \bm S^* \cap \bm S^{\mathrm{opt}}$, and set $\beta=\frac{\alpha'}{1-\alpha'}$. Since $\alpha'<1/3$, we have $\beta <0.5$. One can verify that $|\bm T| \geq (1-\beta) \max  \left\{|\bm S^*|, |\bm S^{\mathrm{opt}}|\right\}$. Then, from the property of resilience in Definition 2, we have
\[
\left \|\frac{1}{|\bm T|}\sum_{\bm y_i\in \bm T} \bm y_i - \textcolor{black}{\Bar{\bm x}^*} \right \|_2 \leq 3\sigma \sqrt{\beta} \quad \mbox{and} \quad \left \|\frac{1}{|\bm T|}\sum_{\bm y_i\in \bm T} \bm y_i - \bm x^{\mathrm{opt}} \right \|_2 \leq c_1\sqrt{6}\sigma \sqrt{\beta}.
\]

By the triangle inequality, we obtain 
\[
\|\textcolor{black}{\Bar{\bm x}^*}-\bm x^{\mathrm{opt}}\|_2\leq (3+c_1\sqrt{6})\sigma \sqrt{\beta}=(3+c_1\sqrt{6})\sigma \sqrt{{\alpha'}/(1-\alpha')}< (4+3c_1)\sigma \sqrt{\alpha'}=(4+3c_1)\sigma \sqrt{\alpha+\epsilon}.
\]
\textcolor{black}{
Next, note that for any feasible solution of \eqref{obj}, by setting its non-zero $h_i$ to be 1, we can always get corresponding \emph{feasible} $(\hat{\bm h}, \hat{\bm x})$ with $\hat{h}_i \in \{0,1\}$ and $
\hat{\bm x}=\frac{\sum_{i=1}^{n}(1-\hat{h}_i)\bm y_i}{\sum_{i=1}^{n}(1-\hat{h}_i)}$, i.e., $\hat{\bm x}$ is the average of the $\bm y_i$'s corresponding to $\hat{h}_i=0$, and the objective value remains unchanged.
Since $\|\hat{\bm h}\|_0\leq \alpha' n$, following the same proof as above, we also have
\[
\|\Bar{\bm x}^*-\hat{\bm x}\|_2\leq (4+3c_1)\sigma \sqrt{\alpha+\epsilon}.
\]
}
\end{proof}
\begin{remark}
Observe that in Theorem 1, $\epsilon$ controls the error tolerance level, and the lower bound on the required number of samples is $\Omega(\frac{d \log d}{\epsilon})$ which is independent of the corruption level $\alpha$. Previous works (for e.g., cf.~\cite{7782980,diakonikolas2017being,steinhardt2018robust}) do not consider a tolerance level, and in these works the lower bound on the required number of samples is inverse proportional to the fraction of corruption $\alpha$, which blows up as $\alpha\to 0$. Moreover, $\alpha$ is typically unknown in practice.  Specifying $\epsilon$ to control the estimator's error helps us remove the dependence of the number of samples required on the fraction of corruption $\alpha$.
Note that we can recover the results in the form as given by the previous works by setting $\epsilon=O(\alpha)$ in Theorem 1. The following corollary states this result.\\
\end{remark}

\begin{corollary}
Let $P$ be a distribution on $\mathbb{R}^d$ with unknown mean $\bm x^*$ and unknown covariance matrix $\preceq \sigma^2 I$. Let $0<\alpha\leq 0.33$. Given an $\alpha$-fraction corrupted set of $n\geq 3.2\times 10^4\times 160\frac{d \log d}{\alpha}+11.2\times 10^4\times160 \frac{d}{\alpha}$ samples from $P$, and set $c^2_1\geq \frac{3}{2}$ in \eqref{obj}. With probability at least $0.949$, the globally optimal solution $(\bm h^{\mathrm{opt}}, \bm x^{\mathrm{opt}})$ of \eqref{obj} with $h_i^{\mathrm{opt}} \in \{0,1\}$ satisfies $\| \bm x^{\mathrm{opt}}-\textcolor{black}{\bm {x}^*}\|_2 \leq (7.5+3.1c_1)\sigma \sqrt{\alpha}$, \textcolor{black}{and any feasible solution $(\hat{\bm h}, \hat{\bm x})$ with ${\hat{h}_i} \in \{0,1\}$ and $\|\hat{\bm h}\|_0 \leq \frac{161}{160}\alpha n$ satisfies $\| \hat{\bm x}-{ {\bm x}^*}\|_2 \leq (7.5+3.1c_1)\sigma \sqrt{\alpha}$.}
\end{corollary}
\begin{proof}

Set $\epsilon=\alpha/160$.
On the event $\tilde E$ (as defined in the proof of Theorem 1), we have $\lambda_{\mathrm{max}}\left(\sum_{\tilde{\bm y}_i\in \bm G} (\tilde{\bm y}_i-\bm x^*)(\tilde{\bm y}_i-\bm x^*)^\top\right)\leq \frac{3}{2} n  \sigma^2$. Further, let $\tilde{\bm x}$ be the average of samples in $\bm G$, then we have $\lambda_{\mathrm{max}}\left(\sum_{\tilde{\bm y}_i\in \bm G} (\tilde{\bm y}_i-\tilde{\bm x})(\tilde{\bm y}_i-\tilde{\bm x})^\top\right) \leq \lambda_{\mathrm{max}}\left(\sum_{\tilde{\bm y}_i\in \bm G} (\tilde{\bm y}_i-\bm x^*)(\tilde{\bm y}_i-\bm x^*)^\top\right)\leq \frac{3}{2} n  \sigma^2$. Then, the set $\bm G$ is $(8\sqrt{15/159}\sigma \sqrt{\beta},\beta )$-resilient in $\ell_2$-norm around $\tilde{\bm x}$ for all $\beta<0.5$ by Lemma~\ref{boundedcovariance_resilience}. Using this resilience property and the fact that $\{\bm y_i | h_i^*=0\}\subseteq \bm G$, we have $\|\textcolor{black}{ \Bar{\bm x}^*}-\tilde{\bm x}\|_2\leq 8\sqrt{15/159}\sigma \sqrt{\alpha/(1-\alpha/160)}\leq \frac{160}{159}\sqrt{6}\sigma \sqrt{\alpha}$. Finally, 
\textcolor{black}{from Theorem 1 and} by triangle inequality, we have $\|\bm x^{\mathrm{opt}}-\tilde{\bm x}\|_2 \leq [\frac{160}{159}\sqrt{6}+(4+3c_1)\sqrt{161/160}]\sigma \sqrt{\alpha}\leq (6.5+3.1c_1)\sigma \sqrt{\alpha}$ \textcolor{black}{as well as $\|\hat{\bm x}-\tilde{\bm x}\|_2 \leq (6.5+3.1c_1)\sigma \sqrt{\alpha}$.} Using Lemma \ref{lem:conc} in Appendix~\ref{sec_lemma2}, we get that with high probability, $\lVert \tilde{\bm x}-\bm x^* \rVert_2 \leq \sigma\sqrt{\alpha}$. Consequently, by taking intersection of the events in \eqref{eq:yx}, \eqref{eq:z} and \eqref{eq:G} from Lemma \ref{lem:conc}, and the event $\tilde{E}$ and applying triangle inequality, we obtain that with probability at least 0.949, $\| \bm x^{\mathrm{opt}}-\textcolor{black}{\bm {x}^*}\|_2 \leq (7.5+3.1c_1)\sigma \sqrt{\alpha}$ \textcolor{black}{and $\|\hat{\bm x}-\bm {x}^*\|_2 \leq (7.5+3.1c_1)\sigma \sqrt{\alpha}$.}

\end{proof}

\section{Algorithm}
In this section, we first provide a general framework for solving~\eqref{obj} by alternately updating the outlier indicator vector $\bm h$ and the estimate of the mean $\bm x$. \textcolor{black}{Note that the objective~\eqref{obj} is non-convex. Fortunately, it is known that several efficient algorithms like coordinate-wise median and geometric median can tolerate nearly half outlier points and their estimates are bounded from the true mean. So we can use them as a good initial point $\bm x^{(0)}$ in our algorithm. From Theorem 1, we know that actually any feasible solutions $\bm h$  (with ${\hat{h}_i} \in \{0,1\}$) which are sparse enough, would be sufficient.}

Since updating $\bm h$ (i.e., minimizing $\|\bm h\|_0$ under the constraints in step 1 in Algorithm 1) is computationally expensive, we propose to minimize the surrogate functions $\|\bm h\|_p^p$ with $0<p\leq 1$. The effectiveness of this approach is well understood in the Compressive Sensing literature. 
\subsection{General framework}
Our general framework for solving \eqref{obj} is detailed in Algorithm 1. In Step 1 of Algorithm 1, we fix the current estimate of the mean $\bm x$ and estimate the set of outlier points (corresponding to $h_i \neq0$). In Step 2, we update $\bm x$ as the average of the set of estimated inlier points. Then we repeat this procedure until the stopping criteria is met. The following theorem shows that the objective value is non-increasing through the course of the iterations of this alternating minimization algorithm.

\begin{table}[h]
  \label{Algorithm 1}
  \centering
  \begin{tabular}{lll}
    \toprule
    \multicolumn{1}{c}{\textbf{Algorithm 1} Robust Mean Estimation via $\ell_0$ Minimization}                   \\
    \midrule
    \textbf{Input:} Observations $\bm y_i, i=1,2,...,n$, upper bound \RED{$c^2_1 \sigma^2$}\\
    \textbf{Initialize:}    $\bm x^{(0)}$ as the Coordinate-wise Median of $\bm y_i, i=1,2,...,n$; iteration number $t=0$\\

    \textbf{While $\|\bm h^{(t)}\|_0 < \|\bm h^{(t-1)}\|_0$}        \\
 Step 1: Fix $\bm x^{(t)}$, update $\bm h$ \\

  $\bm h^{(t)}= \arg \min_{\bm h}  \|\bm h\|_0 $,
    s.t. $\  0\leq h_i \leq 1, \forall i$,
  $\lambda_{\mathrm{max}}(\sum_{i=1}^{n} (1-h_i)(\bm y_i-\bm x^{(t)})(\bm y_i-\bm x^{(t)})^\top)\leq c^2_1 n \sigma^2 $\\
 
  Step 2: Fix $\bm h^{(t)}$, update $\bm x$\\
   $\bm x^{(t+1)}=\frac{\sum_{\{i:h_i^{(t)}=0\}} \bm y_i}{|\{i:h_i^{(t)}=0\}|}$,\\

  $t:=t+1$\\
    \textbf{End While}\\
\textbf{Output:} $ \bm x$\\
    \bottomrule
  \end{tabular}
\end{table}

\begin{theorem}\label{converge}
Through consecutive iterations of Algorithm 1 the objective value in \eqref{obj} is non-increasing, and the algorithm  terminates in at most $n$ iterations.
\end{theorem}
\textit{Proof}:
In the $t$-th iteration of Algorithm 1, let $\bm h^{(t)}$ be the solution of Step 1. We introduce a new variable $\bm s^{(t)}$ defined as follows:
\[
s_i^{(t)}=
	\begin{cases}
			0 & \text{if } h_i^{(t)}=0,\\
			1 & \text{otherwise.}
	\end{cases}
\]
Note that we have
\begin{align}\label{obj_t}
\|\bm s^{(t)}\|_0=\|\bm h^{(t)}\|_0. 
\end{align}
Observe that $(\bm y_i-\bm x^{(t)})(\bm y_i-\bm x^{(t)})^\top$ is a PSD matrix for any $i$, so we have
\begin{align}\label{lambda}
\lambda_{\mathrm{max}}(\sum_{i=1}^{n} (1-s_i^{(t)})(\bm y_i-\bm x^{(t)})(\bm y_i-\bm x^{(t)})^\top) \leq \lambda_{\mathrm{max}}(\sum_{i=1}^{n} (1-h_i^{(t)})(\bm y_i-\bm x^{(t)})(\bm y_i-\bm x^{(t)})^\top). 
\end{align}
Since the index sets satisfy $\{i:s_i^{(t)}=0\}=\{i:h_i^{(t)}=0\}$, we get that the output in Step 2 satisfies, 
\[
\bm x^{(t+1)}=\frac{\sum_{\{i:h_i^{(t)}=0\}} \bm y_i}{|\{i:h_i^{(t)}=0\}|}=\frac{\sum_{\{i:s_i^{(t)}=0\}} \bm y_i}{|\{i:s_i^{(t)}=0\}|}. 
\]
As all the nonzero values of $\bm s^{(t)}$ are 1, it follows that $x^{(t+1)}$ is the optimum solution of $\min_{\bm x}\ \lambda_{\mathrm{max}}(\sum_{i=1}^{n} (1-s_i^{(t)})(\bm y_i-\bm x)(\bm y_i-\bm x)^\top)$. Thus, we must have 
\[
\lambda_{\mathrm{max}}(\sum_{i=1}^{n} (1-s_i^{(t)})(\bm y_i-\bm x^{(t+1)})(\bm y_i-\bm x^{(t+1)})^\top) \leq \lambda_{\mathrm{max}}(\sum_{i=1}^{n} (1-s_i^{(t)})(\bm y_i-\bm x^{(t)})(\bm y_i-\bm x^{(t)})^\top). 
\]
Applying \eqref{lambda}, we obtain
\[
\lambda_{\mathrm{max}}(\sum_{i=1}^{n} (1-s_i^{(t)})(\bm y_i-\bm x^{(t+1)})(\bm y_i-\bm x^{(t+1)})^\top) \leq \lambda_{\mathrm{max}}(\sum_{i=1}^{n} (1-h_i^{(t)})(\bm y_i-\bm x^{(t)})(\bm y_i-\bm x^{(t)})^\top).
\]
Since $\{\bm x^{(t)}, \bm h^{(t)} \}$ is a feasible solution of \eqref{obj}, we get that $\{\bm x^{(t+1)}, \bm s^{(t)} \}$ must also be a feasible solution. Since $\bm h^{(t+1)}$ is the optimal solution of Step 1 in iteration $t+1$ with $\bm x$ fixed as $\bm x^{(t+1)}$, we obtain $\|\bm h^{(t+1)}\|_0 \leq \|\bm s^{(t)}\|_0$. Consequently, it follows from \eqref{obj_t} that $\|\bm h^{(t+1)}\|_0 \leq \|\bm h^{(t)}\|_0$. Because $\|\bm h^{(0)}\|_0 \leq n$ and the objective value in \eqref{obj} is always non-negative, Algorithm 1 will terminate in at most $n$ iterations. 
\hfill\qedsymbol

\begin{remark}\label{rem:2}
Note that Theorem 2 does not establish convergence to the global optimal solution of~\eqref{obj}. It only guarantees the decreasing of the objective value. However, Theorem 1 states that with probability at least 0.95, any feasible solution $(\hat{\bm h}, \hat{\bm x})$ with ${\hat{h}_i} \in \{0,1\}$ and $\|\hat{\bm h}\|_0 \leq (\alpha+\epsilon)n$ satisfies $\| \hat{\bm x}-{ \Bar{\bm x}^*}\|_2 \leq (4+3c_1)\sigma \sqrt{\alpha+\epsilon}$. It is not necessary to reach the global optimum of the objective~\eqref{obj}.
\end{remark}

\subsection{Solving Step 1 of Algorithm 1}
\label{headings}

The $\ell_0$ minimization problem in Step 1 of Algorithm 1 is computationally challenging in general. Motivated by the success of the $\ell_1$ and $\ell_p \ (0<p<1)$ techniques in Compressive Sensing, we use $\|\bm h\|_p^p$ (with $0<p\leq 1$) as surrogate functions for $\|\bm h\|_0$ in the minimization. When $p=1$, the problem is convex, and can be reformulated as the following packing SDP with $w_i\triangleq 1-h_i$, and $e_i$ being the $i$-th standard basis vector in $\mathbb{R}^n$. The details can be found in the Appendix:
\begin{align}\label{packing SDP}
   \max_{\bm w} \ \bm 1^\top \bm w \quad 
    s.t. &\  w_i \geq 0, \forall i\\
   & \sum_{i=1}^{n} w_i \begin{bmatrix}
    e_ie_i^\top &  \\
    &  (\bm y_i-\bm x)(\bm y_i-\bm x)^\top \end{bmatrix} \preceq \begin{bmatrix}
    I_{n\times n} &  \\
    &  c^2_1 n \sigma^2 I_{d\times d} \end{bmatrix} \nonumber
\end{align}

When $0<p<1$, the surrogate function $\|\bm h\|_p^p = \sum_i h_i^p$ is concave. We can iteratively construct and minimize a \textit{tight} upper bound on this surrogate function via iterative re-weighted $\ell_2$~\cite{4518498,gorodnitsky1997sparse} or iterative re-weighted $\ell_1$ techniques~\cite{candes2008enhancing} from Compressive Sensing.\footnote{We observe that iterative re-weighted $\ell_2$ achieves better empirical performance.}

\textbf{Numerical Example.}
We illustrate the effectiveness of $\ell_1$ and $\ell_p \ (0<p<1)$ in approximating $\ell_0$ through the following numerical example. The dimension of the data is $d=100$, and for visualization purposes, we set the number of data points to be $n=200$. The outlier fraction is set to be 10\%. The inlier data points are randomly generated from the standard Gaussian distribution with zero mean. For the outliers, half of them (i.e., 5\%) are set to be $[\sqrt{d/2},\sqrt{d/2},0,...,0]$, and the other half are set as $[\sqrt{d/2},-\sqrt{d/2},0,...,0]$, so that their $\ell_2$ distances to the true mean $[0,...,0]$ are all $\sqrt{d}$, similar to that of the inlier points. We fix $\bm x$ to be coordinate-wise median of the data points, and then solve $\bm h$ via $\ell_1$ minimization or $\ell_p$ minimization with $p=0.5$. Fig.~\ref{fig:LpApproximation} shows an example solution of $\bm h$ by $\ell_1$ method (left, blue dots) and $\ell_p$ method (right, green dots). The red circles in the figure correspond to the true indices of the outlier points. First, we can see that both the $\ell_1$ and $\ell_p$ minimization lead to sparse solutions of $\bm h$, and the solution of $\ell_p$ minimization is even sparser. Further, from their solutions, we can very effectively identify the set of inlier points (corresponding to $h_i=0$). This enables us to accurately estimate $\bm x$ in Step 2 of Algorithm 1.

\begin{figure*}[!h]
	\centering
	\includegraphics[trim=50 0 50 0, clip=true,width=0.8\linewidth]{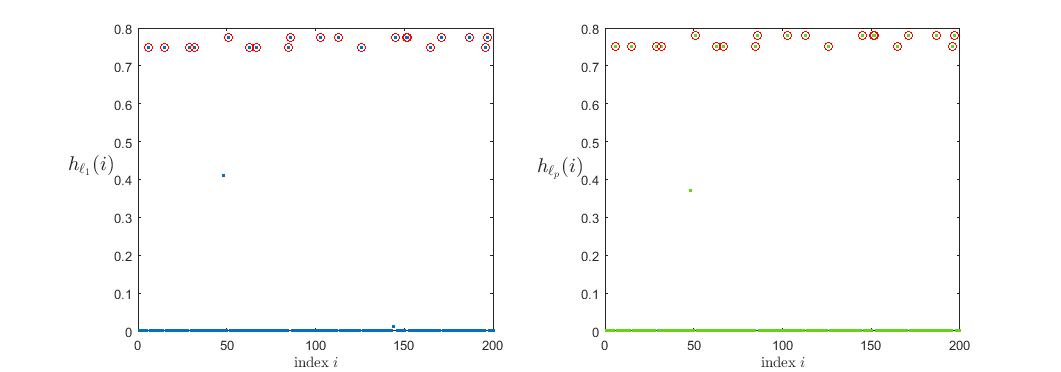}
	\caption{The solution of $\bm h$ by $\ell_1$ method (left, blue dots) and $\ell_p$ method ($p=0.5$, right, green dots). The red circles correspond to the true indices of the outlier points.}
	\label{fig:LpApproximation}
\end{figure*}

\subsection{Complexity analysis}

Theorem~\ref{converge} guarantees that the total number of iterations performed by Algorithm 1 is upper bounded by $n$. We note that in practice Algorithm 1 often terminates in much fewer than $n$ iterations. In each iteration, the computational complexity of Step 2 is $O (nd)$. In Step 1, if we relax $\ell_0$-`norm' to $\ell_1$, we can solve the resulting Packing SDP ~\eqref{packing SDP} to precision $1-O(\epsilon)$ in $\tilde{O}(nd/\epsilon^6)$ parallelizable work using positive SDP solvers~\cite{Allen-Zhu2016} (the notation $\tilde{O}(m)$ hides the poly-log factors: $\tilde{O}(m)=O(m\text{polylog}(m))$).

If we relax the $\ell_0$-`norm' to $\ell_p$ with $0<p<1$ in Step 1, we iteratively construct and minimize a \textit{tight} upper bound on the $\ell_p$ objective via iterative re-weighted $\ell_2$~\cite{4518498,gorodnitsky1997sparse} or iterative re-weighted $\ell_1$ techniques~\cite{candes2008enhancing}\footnote{We only run a few re-weighted iterations in our implementation.}. Minimizing the resulting weighted $\ell_1$ objective can be also solved very efficiently to precision $1-O(\epsilon)$ by formulating it as a Packing SDP (see 
Appendix) with computational complexity $\tilde{O}(nd/\epsilon^6)$~\cite{Allen-Zhu2016}. If using iterative re-weighted $\ell_2$, minimizing the resulting weighted $\ell_2$ objective is \RED{a SDP constrained least squares problem, whose computational complexity is in general polynomial in both $d$ and $n$.} We will explore more  efficient solutions to this objective in future work.

\section{Empirical Studies}
\label{others}
In this section, we present empirical results on the performance of the proposed methods and compare with the following state-of-the-art high dimension robust mean estimation methods: Iterative Filtering~\cite{diakonikolas2017being}, Quantum Entropy Scoring (QUE)~\cite{NIPS2019_8839}, which scores the outliers based on multiple directions. Note that the above methods as well as the proposed methods need to specify the upper bound on the spectral norm of the sample covariance matrix of the inlier points. We use the source codes from the authors and tune the parameters carefully. Throughout the experiments, we fix $p=0.5$ for the proposed $\ell_p$ method. We also test the method proposed by Lai et al.~\cite{LRV} (denoted as LRV), which needs the knowledge of outlier fractions. We additionally compare with a recently proposed method by Cheng et al.~\cite{convex} (denoted as CDG), which needs the knowledge of \textit{both} outlier fraction and the upper bound on the spectral norm of inlier covariance matrix. The true values are provided to this method. For evaluation, we define the recovery error as the $\ell_2$ distance of the estimated mean to the oracle solution, i.e., the average of the uncorrupted samples.

\subsection{Synthetic data}
We conducted experiments on two different settings of outliers, descirbed below, where the dimension of the data is $d$, and the number of data points is $n$:

\textbf{Setting A}: In this setting, there is one cluster of outliers where their $\ell_2$ distances to the true mean $\bm x$ are similar to that of the inlier points. Besides that, there are also some other outliers that have different distances to the true mean. More specifically, the inlier points are randomly generated from $\mathcal{N}(0,\, \bm I_{d\times d})$. Half of the outliers are generated from $|\mathcal{N}(0,\, \bm I_{d\times d})|$, where the absolute sign applies on all the $d$ entries. The other half of the outliers are generated by adding corruptions to each entry of the points generated from $\mathcal{N}(0,\, \bm I_{d\times d})$, where the values of the corruptions are randomly drawn from $U(0,3)$. We vary the total fraction $\alpha$ of the outliers and report the average recovery error of each method over 20 trials in Table~\ref{settingA_400} with $d=400, n=800$, and in Table~\ref{settingA_100} with $d=100, n=2000$.
\begin{table}[!h]\centering
	\caption{Recovery error of each method under different fraction $\alpha$ of the outlier points (Setting A, $d=400, n=800$)}
	\label{settingA_400}
	\begin{tabular}{|c|c|c|c|c|c|c|c|}
		
		\hline
		$\alpha$	   & Iter Filter & QUE & LRV & CDG & $\ell_1$ & $\ell_p$ \\
		\hline
		10\% & 0.2906 & 1.030 & 0.4859 & 0.0676 & \textbf{0.0185} & \textbf{0.0052} \\
		\hline
		20\% & 0.3532 & 1.162 & 0.7698 & 0.0878 & \textbf{0.0277} & \textbf{0.0148} \\
		\hline
		30\% & 0.4716 & 1.285 & 1.164 & 0.1123 & \textbf{0.0282} & \textbf{0.0188}  \\
		\hline
	\end{tabular}
\end{table}
\begin{table}[!h]\centering
	\caption{Recovery error of each method under different fraction $\alpha$ of the outlier points (Setting A, $d=100, n=2000$)}
	\label{settingA_100}
	\begin{tabular}{|c|c|c|c|c|c|c|c|}
		
		\hline
		$\alpha$	   & Iter Filter & QUE & LRV & CDG & $\ell_1$ & $\ell_p$ \\
		\hline
		10\% & 0.1058 & 0.4451 & 0.1651 & 0.0356 & \textbf{0.0302} & \textbf{0.0287} \\
		\hline
		20\% & 0.1615 & 0.5046 & 0.2664 & 0.0488 & \textbf{0.0358} & \textbf{0.0335} \\
		\hline
		30\% & 0.2731 & 1.138 & 0.3596 & 0.0613 & \textbf{0.0429} & \textbf{0.0422}  \\
		\hline
	\end{tabular}
\end{table}

It can be seen that the proposed $\ell_1$ and $\ell_p$ methods demonstrate much better recovery performance than the state-of-the-art approaches.

\textbf{Setting B}: In this setting, there are two clusters of outliers, and their $\ell_2$ distances to the true mean $\bm x$ are similar to that of the inlier points. The inlier data points are randomly generated from the standard Gaussian distribution with zero mean. For the outliers, half of them are set to be $[\sqrt{d/2},\sqrt{d/2},0,...,0]$, and the other half are set as $[\sqrt{d/2},-\sqrt{d/2},0,...,0]$, so that their $\ell_2$ distances to the true mean $[0,...,0]$ are all $\sqrt{d}$, similar to that of the inlier points. We vary the total fraction $\alpha$ of the outliers and report the average recovery error of each method over 20 trials in Table~\ref{settingB_100} with $d=100, n=2000$. The proposed $\ell_1$ and $\ell_p$ methods show significant improvements over the competing methods, and $\ell_p$ performs the best.

\begin{table}[!h]\centering
	\caption{Recovery error of each method under different fraction $\alpha$ of the outlier points (Setting B, $d=100, n=2000$)}
	\label{settingB_100}
	\begin{tabular}{|c|c|c|c|c|c|c|c|}
		
		\hline
		$\alpha$	   & Iter Filter & QUE & LRV & CDG & $\ell_1$ & $\ell_p$ \\
		\hline
		10\% & 0.0865 & 0.7728 & 0.2448 & 0.0329 & \textbf{0.0123} & \textbf{0.069} \\
		\hline
		20\% & 0.0892 & 0.4971 & 0.4962 & 0.0437 & \textbf{0.0127 }& \textbf{0.0092} \\
		\hline
		30\% & 0.0912 & 0.5076 & 0.8708 & 0.1691 & \textbf{0.0173} &\textbf{ 0.0132 } \\
		\hline
	\end{tabular}
\end{table}

Lastly, we tested the performance of each method w.r.t. different number of samples. The dimension of the data is fixed to be 100. The fraction of the corrupted points is fixed to be 20\%, and the data points are generated as per Setting B. We vary the number of samples from 100 to 5000, and report the recovery error of each method in Table~\ref{settingB_100_02}. We can see that the performance of each method gets better when the number of samples is increased. Again, the proposed methods perform the best under various number of samples.
\begin{table}[!h]\centering
	\caption{Recovery error of each method w.r.t. different number of samples (Setting B, $d=100, \alpha=0.2$)}
	\label{settingB_100_02}
	\begin{tabular}{|c|c|c|c|c|c|c|c|}
		
		\hline
		\# samples	   & Iter Filter & QUE & LRV & CDG & $\ell_1$ & $\ell_p$ \\
		\hline
		100 & 0.4817 & 1.684 & 1.408 & 0.3078 & \textbf{0.2253} & \textbf{0.2100} \\
		\hline
		200 & 0.3246 & 1.175 & 1.077 & 0.1824 & \textbf{0.1124} & \textbf{0.0815} \\
		\hline
		500 & 0.1846 & 0.8254 & 0.7506 & 0.1437 & \textbf{0.0482} & \textbf{0.0337}\\
		\hline
		1000 & 0.1307 & 0.6366 & 0.6244 & 0.1051 & \textbf{0.0255} & \textbf{0.0180 } \\
		\hline
		2000 & 0.0892 & 0.4971 & 0.4962 & 0.0437 & \textbf{0.0127 }& \textbf{0.0092} \\
		\hline
		5000 & 0.0598 & 0.4076 & 0.3892 & 0.0170 & \textbf{0.0064} & \textbf{0.0051}  \\
		\hline
	\end{tabular}
\end{table}

\subsection{Real data}

In this section, we use real face images to test the effectiveness of the robust mean estimation methods. The average face of particular regions or certain groups of people is useful for many social and psychological studies~\cite{little2011facial}. Here we use 140 face images from Brazilian face database\footnote{https://fei.edu.br/~cet/facedatabase.html}, where 100 of them are well-controlled frontal faces with neutral expressions, which are considered to be inliers. The rest of 40 images either have large poses of the head, or have smiling expressions and upside down, which are considered to be outliers. The size of the face images is 36 $\times$ 30, so the dimension of each data point is 1080. The oracle solution is the average of the inlier faces. Table~\ref{face} reports the recovery error, which is the $\ell_2$ distance of the estimated mean face to the oracle solution, by each method. The proposed methods achieve much smaller recovery error than the state-of-the-art methods. The sample face images and the reconstructed mean face images by each method can be found in the Appendix. 

\begin{table}[!h]\centering
	\caption{Recovery error of the mean face by each method}
	\label{face}
	\begin{tabular}{|c|c|c|c|c|c|c|c|}
		
		\hline
		Mean &\shortstack{coordinate-wise \\ median} & Iter Filter & QUE & LRV & $\ell_1$ & $\ell_p$ \\
		\hline
		667 & 250 & 228 & 234 & 439 & \textbf{73} & \textbf{18} \\
		\hline

	\end{tabular}
\end{table}

\section{Conclusion} 


We formulated the robust mean estimation as the minimization of  the  $\ell_0$-`norm' of the introduced \emph{outlier indicator vector}, under second moment constraints on the inlier points. We replaced $\ell_0$ by $\ell_p$ $(0<p\leq 1)$ to provide computationally tractable solutions as in CS, and showed that these solutions significantly outperform state-of-the-art robust mean estimation methods. We observed strong numerical evidence that $\ell_p$ $(0<p\leq 1)$ leads to sparse solutions; theoretically justifying this phenomenon is ongoing work. Along these lines, two recent works~\cite{cheng2020high,zhu2020robust} show that any approximate stationary point of the objective in~\cite{convex} gives a near-optimal solution. It is of interest to see if a similar property can be shown for the proposed $\ell_0$ and  $\ell_p$ $(0<p\leq 1)$ objectives.
\bibliographystyle{IEEEtran}

\bibliography{main}
\newpage
\section{Appendix}

\subsection{Solving $\ell_1$ objective via Packing SDP}
\begin{align}\label{obj1}
   \min_{ \bm h} & \|\bm h\|_1 \\
    s.t. \ & 0\leq h_i \leq 1, \forall i, \nonumber\\
   & \lambda_{\mathrm{max}} \left(\sum_{i=1}^{n} (1-h_i)(\bm y_i-\bm x)(\bm y_i-\bm x)^\top\right)\leq c n \sigma^2. \nonumber
\end{align}

Define the vector $\bm w$ with $w_i\triangleq 1-h_i$. Since $0\leq h_i \leq 1$, we have $0\leq w_i \leq 1$. Further, $\| \bm h \|_1= \sum_{i=1}^{n} h_i=\sum_{i=1}^{n} (1-w_i)=n- \sum_{i=1}^{n} w_i =n- \bm 1^\top \bm w$. Therefore, solving \eqref{obj1} is equivalent to solving the following:
\begin{align}\label{obj1w}
   \max_{\bm w} & \ \bm 1^\top \bm w \\
    s.t. \ & 0\leq w_i \leq 1, \forall i, \nonumber\\
   & \lambda_{\mathrm{max}}\left(\sum_{i=1}^{n} w_i(\bm y_i-\bm x)(\bm y_i-\bm x)^\top\right)\leq c n \sigma^2. \nonumber
\end{align}

Then, we rewrite the constraints $0 \leq w_i \leq 1, \forall i$ as $0 \leq w_i$, and $\sum w_i e_ie_i^\top \preceq I_{n\times n}$, where $e_i$ is the $i$-th standard basis vector in $\mathbb{R}^n$. 
This establishes the equivalence between \eqref{obj1w} and \eqref{packing SDP}.

\subsection{Minimizing $\ell_p$ via iterative re-weighted $\ell_2$}
Consider the relaxation of $\ell_0$ to $\ell_p$ ($0<p<1$) in Step 1 of Algorithm 1. We have the following objective:
\begin{align}\label{objp}
   \min_{ \bm h} & \|\bm h\|_p^p \\
    s.t. \ & 0\leq h_i \leq 1, \forall i, \nonumber\\
   & \lambda_{\mathrm{max}}\left(\sum_{i=1}^{n} (1-h_i)(\bm y_i-\bm x)(\bm y_i-\bm x)^\top\right)\leq c n \sigma^2. \nonumber
\end{align}
Note that $\|\bm h\|_p^p=\sum_{i=1}^{n} h_i^p=\sum_{i=1}^{n} (h_i^2)^{\frac{p}{2}}$. Consider that we employ the iterative re-weighted $\ell_2$ technique~\cite{4518498,gorodnitsky1997sparse}. Then at $(k+1)$-th inner iteration, we construct a tight upper bound on $\|\bm h\|_p^p$ at ${\bm h^{(k)}}^2$ as
\begin{equation}
    \sum_{i=1}^{n} \left\lbrack{\left({h_i^{(k)}}^2\right)}^{\frac{p}{2}}+\frac{p}{2}{\left({h_i^{(k)}}^2\right)}^{\frac{p}{2}-1}\left(h_i^2-{h_i^{(k)}}^2\right)\right\rbrack.
\end{equation} 
We minimize this upper bound:
\begin{align}\label{objp_t}
   \min_{ \bm h} & \sum_{i=1}^{n}{\left({h_i^{(k)}}^2\right)}^{\frac{p}{2}-1}h_i^2 \\
    s.t. \ & 0\leq h_i \leq 1, \forall i, \nonumber\\
   & \lambda_{\mathrm{max}}\left(\sum_{i=1}^{n} (1-h_i)(\bm y_i-\bm x)(\bm y_i-\bm x)^\top\right)\leq c n \sigma^2, \nonumber
\end{align}
Define $u_i={\left({h_i^{(k)}}\right)}^{\frac{p}{2}-1}$, the objective in~\eqref{objp_t} becomes $\sum_{i=1}^{n} u_i^2h_i^2$. Define the vector $\bm w$ with $w_i\triangleq 1-h_i$. Since $0\leq h_i \leq 1$, we have $0\leq w_i \leq 1$. Further, $\sum_{i=1}^{n} u_i^2 h_i^2= \sum_{i=1}^{n} u_i^2 (1-w_i)^2=\sum_{i=1}^{n} (u_i-u_iw_i)^2$. So, solving~\eqref{objp_t} is equivalent to solving the following:
\begin{align}\label{objp_t_w}
   \min_{ \bm w} & \sum_{i=1}^{n}(u_i-u_iw_i)^2 \\
    s.t. \ & 0\leq w_i \leq 1, \forall i, \nonumber\\
   & \lambda_{\mathrm{max}}(\sum_{i=1}^{n} w_i(\bm y_i-\bm x)(\bm y_i-\bm x)^\top)\leq c n \sigma^2. \nonumber
\end{align}

Further, define the vector $\bm z$ with $z_i\triangleq u_iw_i$. Then solving~\eqref{objp_t_w} is equivalent to solving the following:
\begin{align}\label{objp_t_z}
   \min_{ \bm z} & \| \bm u-\bm z\|_2^2 \\
    s.t. \ & 0\leq z_i \leq u_i, \forall i, \nonumber\\
   & \lambda_{\mathrm{max}}\left (\sum_{i=1}^{n} z_i[(\bm y_i-\bm x)(\bm y_i-\bm x)^\top/u_i]\right)\leq c n \sigma^2. \nonumber
\end{align}

Then, we rewrite the constraints $0 \leq z_i \leq u_i, \forall i$ as $0 \leq z_i$, and $\sum_{i=1}^{n} z_i e_ie_i^\top \preceq \diag (\bm u)$, where $e_i$ is the $i$-th standard basis vector in $\mathbb{R}^n$. Finally, we can turn~\eqref{objp_t_z} into the following least squares problem with semidefinite cone constraints:
\begin{align}\label{objp_t_z_SDP}
   \min_{ \bm z} & \| \bm u-\bm z\|_2^2 \\
    s.t. \ & \ z_i \geq 0,  \forall i, \nonumber\\
   & \sum_{i=1}^{n} z_i \begin{bmatrix}
    e_ie_i^\top &  \\
    &  (\bm y_i-\bm x)(\bm y_i-\bm x)^\top /u_i \end{bmatrix} \preceq \begin{bmatrix}
    \diag (\bm u)&  \\
    &  c n \sigma^2 I_{d\times d} \end{bmatrix} .\nonumber
\end{align}



\subsection{Solving weighted $\ell_1$ objective via Packing SDP}
Consider the relaxation of $\ell_0$ to $\ell_p$ $(0<p<1)$ in Step 1 of Algorithm 1 (i.e., minimizing $\|\bm h\|_p^p$). If we employ iterative  re-weighted $\ell_1$ approach~\cite{candes2008enhancing,4518498}, we need to solve the following problem:
\begin{align}\label{objp_wL1}
   \min_{ \bm h} & \sum_{i=1}^{n} u_i h_i\\
    s.t. \ & 0\leq h_i \leq 1, \forall i, \nonumber\\
   & \lambda_{\mathrm{max}}\left(\sum_{i=1}^{n} (1-h_i)(\bm y_i-\bm x)(\bm y_i-\bm x)^\top\right)\leq c n \sigma^2, \nonumber
\end{align}

where $u_i$ is the weight on corresponding $h_i$. Define the vector $\bm w$ with $w_i\triangleq 1-h_i$. Since $0\leq h_i \leq 1$, we have $0\leq w_i \leq 1$. Further, $\sum_{i=1}^{n} u_i h_i= \sum_{i=1}^{n} u_i (1-w_i)=\sum_{i=1}^{n} u_i- \sum_{i=1}^{n} u_i w_i$. So, solving \eqref{objp_wL1} is equivalent to solving the following:
\begin{align}\label{obj1wp}
   \max_{\bm w} & \ \bm u^\top \bm w \\
    s.t. \ & 0\leq w_i \leq 1, \forall i, \nonumber\\
   & \lambda_{\mathrm{max}}\left(\sum_{i=1}^{n} w_i(\bm y_i-\bm x)(\bm y_i-\bm x)^\top\right)\leq c n \sigma^2. \nonumber
\end{align}

Then, we rewrite the constraints $0 \leq w_i \leq 1, \forall i$ as $0 \leq w_i$, and $\sum w_i e_ie_i^\top \preceq I_{n\times n}$, where $e_i$ is the $i$-th standard basis vector in $\mathbb{R}^n$. Finally, we can turn it into the following Packing SDP:
\begin{align}\label{packing SDPw}
   \max_{\bm w} & \ \bm u^\top \bm w \\
    s.t. \ & w_i \geq 0, \forall i, \nonumber\\
   & \sum_{i=1}^{n} w_i \begin{bmatrix}
    e_ie_i^\top &  \\
    &  (\bm y_i-\bm x)(\bm y_i-\bm x)^\top \end{bmatrix} \preceq \begin{bmatrix}
    I_{n\times n} &  \\
    &  c n \sigma^2 I_{d\times d} \end{bmatrix} .\nonumber
\end{align}

\subsection{Lemma 2}\label{sec_lemma2}
\begin{lemma}\label{lem:conc} Let $0<\alpha<\frac{1}{2}$ and $n\geq 9\times 10^3 \frac{d}{\alpha}$.
Let $\tilde{\bm y}_1,\dots,\tilde{\bm y}_n$ be i.i.d. samples drawn from a distribution with mean $\bm x^*$ and covariance matrix $\Sigma \preccurlyeq \sigma^2 I$.
Let $\bm G=\lbrace \tilde{\bm y}_i: \lVert \tilde{y}_i-\bm x^* \rVert_2\leq 80\sigma\sqrt{d/\alpha} \rbrace$. Let $\tilde{\bm x}$ be the mean of samples in $\bm G$. Then the following holds with probability at least 0.974:
\begin{equation}
    \lVert \tilde{\bm x}-\bm x^* \rVert_2 \leq \sigma\sqrt{\alpha}.
\end{equation}
\end{lemma}
\begin{proof}
Note that 

\begin{align}
    &\left\lVert \frac{|\bm G|}{n}(\tilde{\bm x} - \bm x^*) \right\rVert_2\\
    =& \left\lVert \frac{1}{n}\sum\limits_{i=1}^n \tilde{\bm y}_i-\bm x^*- \frac{1}{n} \sum\limits_{i=1}^n (\tilde{\bm y}_i - \bm x^*)\mathbbm{1}\lbrace \lVert \tilde{\bm y}_i-\bm x^* \rVert_2> 80\sigma\sqrt{d/\alpha} \rbrace \right\rVert_2\\
    \leq  &\left\lVert \frac{1}{n}\sum\limits_{i=1}^n \tilde{\bm y}_i-\bm x^* \right\rVert_2+\left\lVert \frac{1}{n} \sum\limits_{i=1}^n \bm z_i \right\rVert_2\\
    \leq & \left\lVert \frac{1}{n}\sum\limits_{i=1}^n \tilde{\bm y}_i-\bm x^* \right\rVert_2+\left\lVert \frac{1}{n}\sum\limits_{i=1}^n \bm z_i-E[\bm z] \right\rVert_2 + \left\lVert E[\bm z] \right\rVert_2,\label{eq:start}
\end{align}
where $\bm z_i=(\tilde{\bm y}_i - \bm x^*)\mathbbm{1}\lbrace \lVert \tilde{\bm y}_i-\bm x^* \rVert_2> 80\sigma\sqrt{d/\alpha} \rbrace$. By Markov's inequality, we get
 \begin{align}
     &\left\lVert \frac{1}{n}\sum\limits_{i=1}^n \tilde{\bm y}_i-\bm x^* \right\rVert_2\leq 0.49\sigma\sqrt\alpha \text{ with probability at least } 1-\frac{9d}{2n\alpha} \label{eq:yx}\text{ and,}\\
     &\left\lVert \frac{1}{n}\sum\limits_{i=1}^n \bm z_i-E[\bm z] \right\rVert_2\leq 0.49\sigma\sqrt\alpha \text{ with probability at least } 1-\frac{9d}{2n\alpha}.\label{eq:z}
 \end{align}

Futhermore,
\begin{align}
    \left\lVert E[\bm z]\right\rVert_2=  &\left\lVert E\left[(\tilde{\bm y} - \bm x^*)\mathbbm{1}\lbrace \lVert \tilde{\bm y}-\bm x^* \rVert_2> 80\sigma\sqrt{d/\alpha} \rbrace \right]\right\rVert_2\\
    = & \max_{\lVert v\rVert_2=1}  v^\top E\left[(\tilde{\bm y} - \bm x^*)\mathbbm{1}\lbrace \lVert \tilde{\bm y}-\bm x^* \rVert_2> 80\sigma\sqrt{d/\alpha} \rbrace \right]\\
    = &\max_{\lVert v\rVert_2=1}   E\left[v^\top(\tilde{\bm y} - \bm x^*)\mathbbm{1}\lbrace \lVert \tilde{\bm y}-\bm x^* \rVert_2> 80\sigma\sqrt{d/\alpha} \rbrace \right]\\
    \mathop{\leq}\limits^{\text{(a)}} & \max_{\lVert v\rVert_2=1} \sqrt{E[v^\top (\tilde{\bm y}-\bm x^*)]^2P(\lVert \tilde{\bm y}-\bm x^* \rVert_2> 80\sigma\sqrt{d/\alpha} )}\label{eq:cs}\\
    = &\sqrt{\lambda_{\max}\left(\Sigma \right)P(\lVert \tilde{\bm y}-\bm x^* \rVert_2> 80\sigma\sqrt{d/\alpha} )}\\
    \mathop{\leq}\limits^{\text{(b)}} &\sqrt{\sigma^2. \frac{\alpha}{80^2}}\label{eq:p80}\\
    \leq & \frac{1}{80}\sigma\sqrt{\alpha}.\label{eq:Ez}
\end{align}
The inequality (a) follows from Cauchy-Schwarz inequality, and (b) follows from Markov's inequality.
By Markov's inequality we also have that
\begin{equation}\label{eq:G}
    |\bm G|\geq n-\frac{\alpha}{160}n \,\,\text{ with probability at least }\frac{39}{40}.
\end{equation}
From \eqref{eq:start}, \eqref{eq:yx}, \eqref{eq:z}, \eqref{eq:Ez} and \eqref{eq:G}, we get that with probability at least $\frac{39}{40}-\frac{9d}{n\alpha}\geq 0.974$,
    \begin{equation}
    \lVert \tilde{\bm x}-\bm x^* \rVert_2 \leq \sigma\sqrt{\alpha}.
\end{equation}
\end{proof}

\subsection{Face images}
We use 140 face images from Brazilian face database\footnote{https://fei.edu.br/~cet/facedatabase.html}, where 100 of them are well-controlled frontal faces with neutral expressions, which are considered to be inliers. The rest of 40 images either have non-frontal orientation of the face, or have upside-down smiling expressions, which are considered to be outliers. Fig.~\ref{fig:sample} shows the sample inlier and outlier face images. Fig.~\ref{fig:face_mean} shows the true average face of the inliers (oracle solution) and the estimated mean faces by each method, as well as their $\ell_2$ distances to the oracle solution. The proposed $\ell_1$ and $\ell_p$ methods achieve much smaller recovery error than the state-of-the-art methods. The estimated mean faces by the proposed methods also look visually very similar to the oracle solution, which illustrates the efficacy of the proposed $\ell_1$ and $\ell_p$ methods. 

\begin{figure*}[!ht]
	\centering
	\includegraphics[trim=0 28 0 20, clip=true,width=1\linewidth]{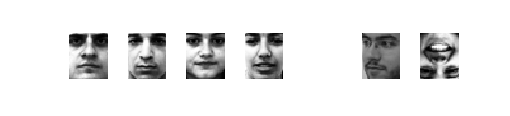}
	\caption{Sample inlier (left 4) and outlier (right 2) face images.}
	\label{fig:sample}
\end{figure*}

\begin{figure*}[h]
	\centering
	\includegraphics[trim=49 78 48 55, clip=true,width=1\linewidth]{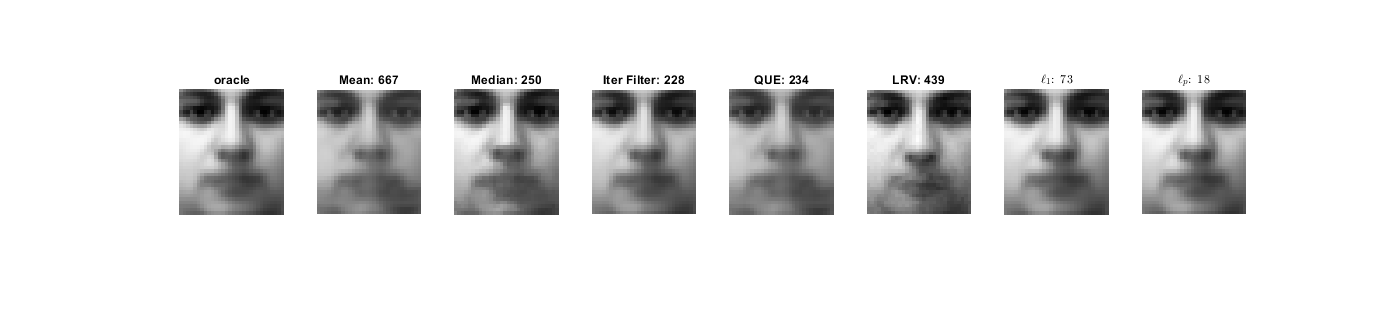}
	\caption{Reconstructed mean face and its recovery error by each method.}
	\label{fig:face_mean}
\end{figure*}

\subsection{Additional experiments}
In this subsection, we test the performance of Iterative Filtering, QUE, LRV, and the proposed $\ell_1$ method under much higher dimension setting of the data as given in Setting B. More specifically, we have $d=2000, n=4000$. Table~\ref{settingB_2000} shows the average recovery error of each method w.r.t. the fraction $\alpha$ of the outlier points. It is evident that the proposed $\ell_1$ method performs much better than the current state-of-the-art methods.
\begin{table}[h]\centering
	\caption{Recovery error of each method under different fraction $\alpha$ of the outlier points (Setting B, $d=2000, n=4000$)}
	\label{settingB_2000}%
	\begin{tabular}{|c|c|c|c|c|c|c|c|}
		
		\hline
		$\alpha$	   & Iter Filter & QUE & LRV &  $\ell_1$ \\
		\hline
		10\% & 0.2713 & 1.055 & 0.5001 &  \textbf{0.0237 } \\
		\hline
		20\% & 0.2828 & 1.148 & 0.9702 & \textbf{0.0256 }  \\
		\hline
		30\% & 0.2851 & 1.321 & 1.9066 &  \textbf{0.0268}  \\
		\hline
	\end{tabular}
\end{table}
\end{document}